\documentclass[letterpaper, 10 pt, conference]{ieeeconf}  % Comment this line out if you need a4paper
\usepackage[english]{babel}
\newtheorem{theorem}{Theorem}
\newtheorem{problem}{Problem}

\usepackage{amsmath}
\usepackage{cite}
\usepackage[hidelinks]{hyperref}
\hypersetup{
    pdftitle={Overleaf Example},
    pdfpagemode=FullScreen,
}
 
\usepackage{bbold}
\usepackage{color,soul}
\usepackage{graphicx}
\usepackage{tikz}
\usepackage{multirow}
\usepackage{subfig}

\usepackage{colortbl}

\usepackage[ruled,linesnumbered,vlined]{algorithm2e}
\SetAlgoCaptionLayout{small} 

\SetCommentSty{mycommfont}
\usepackage{bibentry}
\IEEEoverridecommandlockouts  
\usepackage{caption}% This command is only needed if 
\definecolor{darkpurple}{rgb}{0.4, 0.0, 0.4}

\newcommand{\fmin}{\ensuremath{f_{\min}}\xspace}

\newcommand{\cbest}{\ensuremath{C_{best}}\xspace}

\newcommand{\btit}{{\sc BTIT*}}

\newcommand{\near}{\ensuremath{V_{\text{neighbors}{\text-}\mathcal{F}}}\xspace}

\newcommand{\fminf}{\ensuremath{\widehat{f}_{\min{\text-}\mathcal{F}}}\xspace}

\newcommand{\astate}{\ensuremath{\mathbf{x}}\xspace}

\newcommand{\edge}{\ensuremath{(\mathbf{x},\mathbf{x'})}\xspace}

\newcommand{\si}{\ensuremath{\mathbf{x}_{\rm{start}}}\xspace}
\newcommand{\sg}{\ensuremath{\mathbf{x}_{\rm{goal}}}\xspace}

\newcommand{\ghat}{\ensuremath{{\widehat{g}_{\mathcal{F}}}}}

\newcommand{\hhat}{\ensuremath{\widehat{h}_{_\mathcal{F}}}}
\newcommand{\fhat}{\ensuremath{\widehat{f}_{\mathcal{F}}}}

\newcommand{\fminb}{\ensuremath{\widehat{f}_{\min{\text-}\mathcal{B}}}\xspace}
\title{\LARGE \bf
Optimal Kinodynamic Motion Planning Through Anytime Bidirectional Heuristic Search with Tight Termination Condition 
}

\author{ Yi Wang$^{1}$, Bingxian Mu$^{2*}$, Shahab Shokouhi$^{1}$, and May-Win Thein$^{1}$
\thanks{
$^{1}$ Yi Wang, Shahab Shokouhi, and May-Win Thein are with the Department of Mechanical Engineering, University of New Hampshire, Durham, NH, 03824, USA. Email:{\tt\small \{yw1055@usnh.edu\}, \{shahab.shokouhi, may-win.thein\}@unh.edu} $^{2*}$Bingxian Mu is with the Faculty of Sustainable Design Engineering, University of Prince Edward Island, Charlottetown, PE C1A 4P3, Canada {\em *Corresponding author.} Email: {\tt\small bmu@upei.ca}.
    }
}
\begin{document}

\maketitle
\thispagestyle{empty}
\pagestyle{empty}

%%%%%%%%%%%%%%%%%%%%%%%%%%%%%%%%%%%%%%%%%%%%%%%%%%%%%%%%%%%%%%%%%%%%%%%%%%%%%%%%
\begin{abstract}
%Motion planning is a popular field in robotics. Graph searches and sampling-based planners are usually employed for motion planning. However, these technologies have limitations. 
%due to their superior performance.that lack large local minima 
This paper introduces Bidirectional Tight Informed Trees (\btit), an asymptotically optimal kinodynamic sampling-based motion planning algorithm that integrates an anytime bidirectional heuristic search (Bi-HS) and ensures the \emph{meet-in-the-middle} property (MMP) and optimality (MM-optimality). \btit\ is the first anytime MEET-style algorithm to utilize termination conditions that are efficient to evaluate and enable early termination \emph{on-the-fly} in batch-wise sampling-based motion planning. 
Experiments show that \btit\ achieves strongly faster time-to-first-solution and improved convergence than representative \emph{non-lazy} informed batch planners on two kinodynamic benchmarks: a 4D double-integrator model and a 10D linearized Quadrotor. The source code is available \href{https://github.com/yi213-robotic/Bidirectional-Tight-Informed-Trees}{\textcolor{blue}{here}}.

%In this paper, we investigate an asymptotically optimal sampling-based motion planning algorithm -- Bidirectionally Informed Trees (\btit). Unlike asymmetric bidirectional sampling-based planners, such as Adaptively informed trees (AIT*) and Effort Informed trees (ETI*), \btit provides an accurate and probably admissible heuristic to direct the bidirectional search for an incrementally sampling batch of states that is defined as an increasingly dense implicit random geometric graph (RGGs) to find an asymptotically optimal solution. \btit firstly applies a bidirectional search without collision-checks(lazy bidirectional search) to the implicit random geometric graph (RGGs) until finding a state that satisfies the stop condition similar to Meet-in-Middle(MM). Then, \btit employs Dijkstra's algorithm (without collision-checks) with a bounded estimated solution to update the heuristic for a state generated by the lazy bidirectional search from the state of fulfilling the stop condition towards the start and goal, respectively. Finally, the provided heuristic is used to induct a bidirectional search (with collision-checks) for the given RGGs. It shows that \btit is probabilistically complete and asymptotically optimal.We experimentally present the strengths of \btit on two simulated abstract problems in $\mathbb{R}^8$ and $\mathbb{R}^{16}$. The results show \btit outperforms the existing sampling-based motion planners both in faster finding an initial solution and converging to the optima.  The source code is available at~\href{https://github.com/yi213-robotic/Bidirectional-Tight-Informed-Trees}{\textcolor{blue}{here}}  %
\end{abstract}

%%%%%%%%%%%%%%%%%%%%%%%%%%%%%%%%%%%%%%%%%%%%%%%%%%%%%%%%%%%%%%%%%%%%%%%%%%%%%%%%
\section{INTRODUCTION}
Motion planning is a fundamental problem in robotics and robot navigation. Given a robot and an environment with a set of obstacles, the primary target is to discover a collision-free pat from a specified start to a designated goal in the robot's configuration space~\cite{lavalle2006planning}. A common class of approaches are sampling-based motion-planning (SBMP) algorithms~\cite{PRM,RRTs,rrtStar,Salzman19}, which are particularly effective in high-dimensional continuous configuration spaces. Broadly, these methods interleave two steps: (S1) they approximate the continuous configuration space with a discrete roadmap, an embedded graph whose vertices are sampled configurations and whose edges correspond to collision, checked local connections and (S2) they apply a graph-search algorithm (e.g.,~\cite{dijkstra1959note,A*}) on this roadmap to find a feasible (and often low-cost) path between the start and goal configurations.

The selection of graph-search algorithm is often determined by the structural and validity guarantees provided by the roadmap. For example, if each vertex and edge of the road map have been certified collision-free, then unidirectional heuristic search (Uni-HS) methods, most notably A*~\cite{A*}, become a natural and effective choice. In this setting, a heuristic can incorporate domain knowledge through estimating the remaining \emph{cost-to-go} to the goal, allowing the planner to prioritize promising regions of the graph and substantially reduce unnecessary exploration. Representative alternatives include anytime heuristic planners~\cite{ARA,AWA} for time-limited planning, incremental heuristic planners~\cite{Dlite,koenig2004lifelong} to efficiently address dynamic changes, anytime incremental heuristic planners~\cite{AD*,TA*} that tackle both limited planning time and dynamic settings, and online user-guided planners~\cite{islam2017online} for high-dimensional spaces.

Another family of graph-search algorithms, bidirectional search, is to compute an optimal solution by conducting two separate searches simultaneously from both the start and the goal ~\cite{bis}. Its key strength lies in the potential to exponentially reduce the number of expanded states, as each search only proceeds to half the depth of the path solution~\cite{kof}. This theoretical advantage naturally makes bidirectional heuristic search (Bi-HS) an appealing algorithmic approach.
However, effectively coordinating the meeting of the forward and backward searches presents a significant challenge and has been a long-standing open question using heuristics~\cite{bi1,bi2}. MM~\cite{mm} guarantees finding an optimal solution, and maintains \emph{meet-in-the-middle} property (MM-optimality), such that each expanded state whose \emph{cost-to-come} from the respective search origins never exceeds half of the optimal solution cost. 
While MM and its variants~\cite{mm} guarantee MM-optimality, their termination conditions typically require, at each iteration, checking both the minimal \emph{cost-to-come} and the minimal \emph{total estimated path cost}, which incurs substantial runtime overhead~\cite{MEET}. As a result, designing an efficient tighter termination condition has been a long-standing open question in bidirectional heuristic search~\cite{historyBiHS,MEET}. The most recent advancement, MEET~\cite{MEET}, is the first algorithm to terminate the search with a provably tighter termination condition, and it demonstrates the advantages of MM-optimal algorithms when the heuristics are less informative. 
    
 %\begin{figure}[t]
    %\centering 
%\includegraphics[width=84mm,scale=0.9]{f1.jpeg}
    %\caption{\btit is applied for a path planning problem at the UNH-Hover system (PX4).}
   % \label{fig:drone}  
%\end{figure}

However, graph-based search algorithms can suffer from discretization effects in high-dimensional continuous planning domains: achieving adequate fidelity typically requires very fine discretizations, which rapidly increases computational cost due to the curse of dimensionality~\cite{bellman}. Moreover, any fixed discretization may miss narrow passages or high-quality trajectories, leading to reduced coverage and suboptimal solutions. Improving solution quality generally requires progressively finer resolutions, which further increases runtime and memory demands.

 To efficiently address these difficulties, sampling-based motion planners, including PRM~\cite{PRM}, and RRTs~\cite{RRTs}, avoid the discretization problems of graph-search by randomly sampling states in a high dimensional continuous search space. Additionally, their optimal variants~\cite{rrtStar} find optimal solutions as the number of sampled states approaches to infinity. A key computational bottleneck in many SBMP is  the need to perform collision detection for roadmap vertices and edges. 
To mitigate this, algorithmic approaches such as 
lazy collision checks and informed sampling 
have  significantly improved the efficacy of SBMP~\cite{akgun,rrtconect,FMT,MPLB,informedrrt,LazyPRM*,lazysp,WA_Extend,lra,gls,OrSARon,EMPCbIG}.
Another approach (that may be integrated with the ones mentioned) is to process samples in \emph{batches}. Prominent examples of Batch-wise SBMP (BwSBMP) are BIT*~\cite{bit}, ABIT*~\cite{strub_icra20a}, AIT* and EIT*,~\cite{EIT}. These planners have shown to substantially improve performance when compared to many state-of-the-art planners. An accurate heuristic can substantially augment the anytime profile of a sampling-based motion planner, like AIT* and EIT*~\cite{EIT}. Consequently, we use AIT* and EIT* as our baseline. 

BLIT*~\cite{blit} is the first framework to incorporate anytime, incremental, lazy bidirectional heuristic search (Bi-HS) within BwSBMP. It integrates lazy Bi-HS with a novel lazy edge-evaluation strategy that allocates computational effort to regions near obstacles. Empirically, BLIT* outperforms prior \emph{state-of-the-art} optimal motion planners in high-dimensional continuous configuration spaces, particularly in kinodynamic systems. However, BLIT* does not incorporate a MEET-style termination condition.

To isolate the impact of a tighter termination condition for bidirectional heuristic search in realistic kinodynamic planning settings, we propose Bidirectional Tight Informed Trees (\btit), an anytime batch-wise SBMP planner based on MEET~\cite{MEET} that \emph{disables} lazy edge evaluation. Given an RGG over informed states, \btit\ performs an anytime bidirectional heuristic search (Bi-HS) and processes candidate edge frontiers using a MEET-variant termination condition. In contrast to unidirectional batch-wise planners such as BIT*, \btit\ expands candidate connections from both the start and goal sides and constructs a solution by joining the two search frontiers once they meet. Unlike MEET~\cite{MEET}, \btit\ terminate immediately when the first frontier intersection is detected.

\section{Problem Definition and Notation}
\subsection{Problem Setting}
 In this paper, the kinodynamic optimal planning problem is defined similarly to~\cite{webb2013kinodynamic,kinorrt*}.
 A kinodynamic system with linear (or locally linearized) dynamics is composed of the state space $X\subseteq \mathbb{R}^n$ and control input space $U\subseteq \mathbb{R}^m$. 
  The robot dynamics satisfy the continuous time differential equation:
 \begin{equation}\label{nonlineardynamics}
       \dot{x} = f(x(t),u(t)),x(t)\in X, u(t)\in U.
   \end{equation}
where $t$ denotes time. The equation generates a trajectory $\pi$. To address the kinodynamic motion planning problem, the system dynamics in \eqref{nonlineardynamics} is linearized around a nominal state $x_0(t)$ and control input $u_0(t)$. We evaluate Jacobians of the original system at ($x_0(t)$, $u_0(t)$), 
% \begin{equation*}
% A = \frac{\partial f}{\partial x} \bigg|_{x_0(t), u_0(t)}, \quad B = \frac{\partial f}{\partial u} \bigg|_{x_0(t), u_0(t)}.
% \end{equation*}
and obtain the locally linear (affine) approximation:
\begin{equation}\label{dyna}
\dot{x}(t)=A(t)x(t)+B(t)u(t) + c,
\end{equation} 
where $A \in \mathbb{R}^{n\times n}$ and $B \in \mathbb{R}^{n\times m}$ are the resulting Jacobian matrices, and $c \in \mathbb{R}^n$ is the corresponding offset term. The cost of a trajectory $\pi$ is defined by the function:
\begin{equation}\label{cf}
 c(\pi)= \int_0^\tau (1+u(t)^\top \text{R} u(t)) \text{d}t,
\end{equation} 
where $\tau$ is the trajectory duration and $R \succeq 0$ weights the control effort.

  The $n$-dimensional continuous state space is represented by $X\subseteq \mathbb{R}^n$, $X_{obs} \subset X$ is defined as the set of all states that collide with fixed obstacles, and $X_{free} = X\setminus X_{obs}$ contains states in collision-free space. The start state is denoted by \si\ and the goal state by \sg, where $\{\si,\sg\} \subset X_{free}$. 
  The objective of kinodynamic motion planning is to find a collision-free trajectory $\pi$ that drives the system from \si\ to \sg\ while satisfying the system dynamics, with $x(t)\in X_{free}$ for all $t\in[0,\tau]$.
  
  \begin{problem}[Optimal kinodynamic motion planning] Given a robot configuration space $X$, an obstacle region $X_{\mathrm{obs}}\subset X$ (and thus $X_{\mathrm{free}}=X\setminus X_{\mathrm{obs}}$),
a start state $\si \in X_{\mathrm{free}}$, and a goal region $\sg \subseteq X_{\mathrm{free}}$, let
$\Pi_{\mathrm{free}}$ denote the set of all collision-free trajectories $\pi$ that satisfy the system dynamics and connect
$\si$ to $\sg$. An optimal solution to the problem is to find the trajectory with the minimal cost among
$\Pi_{\mathrm{free}}$,
\[
c^*=c(\pi^*)=\arg\min_{\pi\in\Pi_{\mathrm{free}}} c(\pi),
\]
where $c$ denotes the trajectory cost of $\pi$, and $c^*$ is the cost of the optimal trajectory.
\end{problem}

\subsection{Notation}
Our bidirectional search algorithm maintains two search trees: 
a forward direction rooted at \si and 
a backward direction rooted at \sg.
For brevity, we present the notation and description for the forward direction,~$\mathcal{F}$, with the backward direction, $\mathcal{B}$, being analogous. We use $E$ and $V$ to denote the edge and vertex. 
   For simplicity, our primarily description concentrates on the forward search, with analogous details for the backward search. 

   Let $\mathbf{x}$ be a single state with $\mathbf{x} \in X_{free}$. $\text{prt}_\mathcal{F}(\mathbf{x})$ is the parent state of $\mathbf{x}$ in the forward search.  
   Let $\mathcal{T_F}:=(V_\mathcal{F},E_\mathcal{F})$ be an explicit forward search tree with a set of states, $\{V_\mathcal{F}\} \subseteq X_{free}$, and edges, $E_\mathcal{F} = (\astate,\astate')$,  where \astate\ is the parent of $\astate'$  in the forward motion tree. A {\em priori} heuristic and the lower admissible bound denote the Euclidean Norm through this paper.
   %We note the forward search direction in this section, the backward search direction is analogous.
   
    $\widehat{g}_\mathcal{F}(\astate)$ and $g_{\mathcal{F}}(\astate)$ denote the admissible estimate and the true cost from \si\ to \sg\ given the current forward search tree, respectively. Similarly, $\widehat{c}(\astate,\astate')$ and $c(\astate,\astate')$ denote the admissible estimate and the true cost of the an edge (\astate,$\astate'$), respectively. Finally, 
   $\widehat{h}_\mathcal{F}(\astate)$ and $h_{\mathcal{F}}(\astate)$ show the admissible estimates and the true cost from $\astate$ to $X_{goal}$, respectively.
   
   $\widehat{f}_\mathcal{F}:=\widehat{g}_\mathcal{F}(\astate) + \widehat{c}(\astate,\astate') +\widehat{h}_\mathcal{F}(\astate')$ denotes the admissible estimates of the total path cost from \si\ to \sg\ going through $(\astate,\astate')$. The admissible estimate gives a subset of states, $X_{\widehat{f}}{ := \{ \{\astate, \astate'\} \in X_{free} | \widehat{f}_\mathcal{F} \leq C_{best} \}}$, that could improve the incumbent solution. $\widehat{f}_{min{\mathcal{F}}}$, $\widehat{g}_{min\mathcal{F}}$, and $g_{min\mathcal{F}}$ represent the minimum $\widehat{f}_{\mathcal{F}}$, $\widehat{g}_{\mathcal{F}}$ and $g_\mathcal{F}$, respectively.

   $\text{Q}_\mathcal{F}$ denotes a priority queue storing generated but not yet expanded states in $\mathcal{F}$, ordered by nondecreasing $\widehat{f}$-value. We use~\fminf to denote the minimum $f$-value among all states in $\text{Q}_\mathcal{F}$, and set $\fmin := \min(\fminf,\fminb)$.

   Let \cbest\ denote the true cost of the incumbent solution (i.e., the lowest-cost path found to date). We update \cbest\ whenever a better solution is found. The search proceeds until a user-specified stopping condition is satisfied, with $\cbest$ initialized to $\infty$ at the start of the search.
   For a state that is intersected by a bidirectional heuristic search from both forward and backward directions, we refer to it as an {\em intersecting state}, $\mathbf{\mathcal{I}_b}$.

      In light of the work in~\cite{MEET}, it is readily to show that $ \widehat{g}_\mathcal{F}(\mathbf{x}) \leq g_\mathcal{F}(\mathbf{x})\leq \frac{C_{best}}{2}$. It is also noted that $ \widehat{h}_\mathcal{F}(\mathbf{x}) \leq h_\mathcal{F}(\mathbf{x})$, $h_\mathcal{F}(\mathbf{x}) = g_\mathcal{B}(\mathbf{x})$, and  $\widehat{h}_\mathcal{F}(\mathbf{x}) = \widehat{g}_\mathcal{B} (\mathbf{x})$.

      Finally we will use the notation $A \xleftarrow{+} B $ and $A \xleftarrow{-} B $ to indicate the compounding operations $A \xleftarrow{} A \cup B  $ and $A \xleftarrow{} A\setminus B$, respectively for some sets $A$ and $B$.

\section{Bidirectional Tight Informed Trees (\btit)}
In this section, we sketch Bidirectional Tight Informed Trees, or \btit. At a high level, \btit\ iterates the following steps: (S1) sample a batch of informed states which, connected to the previously sampled states if any, implicitly define an RGG, (S2) run a bidirectional heuristic search over this RGG, and (S3) repeat until the allotted planning time is no longer available. This iterative structure enables the following components, which drive the efficiency of \btit: (C1) update the heuristic \emph{on-the-fly} to better guide subsequent searches, 
(C2) use MEET-style tighter termination condition to stop the search while maintaining the anytime manner, 
and (C3) sample only states that could potentially improve the current best solution.

% \subsection{Neighbors' Connection}
% Similar to~\cite{blit}, \btit\ incrementally approximates the search space, by adding $m$ informed states in each batch, potentially optimizing the current solution (Alg.~\ref{alg1}, line 9). These informed states can be connected with their neighbors that fall within a radius $r$ (-disc RGG) or that are $k$-nearest neighbors (Alg. \ref{alg1}, line 9) to build edge-implicit RGGs~\cite{informedrrt}. %Note that when $\eta > 3^{n}$, the $k$-nearest neighbors are mutual to each other~\cite{FMT}. 
%    %Figure~\ref{illustration} servers as a visual representation, illustrating the workflow of \btit\ throughout the subsequent sections.

\subsection{Algorithmic Description}
  We now present the algorithmic procedures by which \btit\ operates at each iteration.
  \subsubsection{Main Procedure (Alg.~\ref{alg1})}
In the main procedure, \btit\ maintains two trees, a forward tree rooted at $\si$ and a backward tree rooted at $\sg$, constructed over a set of sampled states in the configuration space.
For clarity, we present the forward search operations while the backward search operations follow symmetrically.

   After initialization (Lines 1-2), \btit\ iteratively executes the following steps.  \btit\ begins by checking whether any state remains unexpanded. If none remain, it samples and adds a new batch of sampled states into the configuration space (Lines 4-6). It then selects a state \astate\ for expansion from both $\text{Q}_\mathcal{F}$ and $\text{Q}_\mathcal{B}$, with minimum priority (Line 7).

    After a candidate path is found, \btit\ checks whether the termination condition is satisfied (Line 8) . If this is a case, \btit\ culls states that cannot improve \cbest\ and reset the search queues (Line 15).   
    Otherwise, \btit\ processes the selected state \astate\ (Lines 9-13) until the termination condition is satisfied.     
    
\subsubsection{Expansions Procedure (Alg.~\ref{alg2})}
   We now introduce the expansion process on the forward search tree, while the backward search tree follows analogously.

   This procedure starts by adding \astate\ into the forward tree and removing it from the set of samples (Line 1). Then it iterates over all neighbors of $\astate$ (Line 2). 
   For each neighbor $\astate'$ of \astate, if it is not yet connected to $\mathcal{T}_\mathcal{F}$, \btit\ sets its $g$-value to infinity (Lines 3-4).

   Subsequently, if the edge is already in $\mathcal{T}_\mathcal{F}$, $\mathbf{x'}$ is inserted into~$\text{Q}_\mathcal{F}$ (Lines 5-7). 
   
   If $\astate$ can offer a better solution to a neighbor $\astate'$ and the validity of \edge holds (Line 8-10), \btit\ rewires along \edge\, adds \edge\ into the forward tree, and updates its $g$-value (Line 11 -13). If $\astate'$ intersects both forward and backward trees and a better solution can be provided through $\astate'$, \btit\ updates $\cbest$ (Lines 14-17).  If the first intersection hold during the whole process, \btit\ terminates the current search (Line 18-19). Then \btit\ considers to update $\widehat{h}_\mathcal{F}$-value if needed (Line 20) and $\astate'$ is inserted into $\text{Q}_\mathcal{F}$ if needed (Line 21-22).
\subsubsection{Termination Condition Procedure}\label{cPRBMT}
\btit\ halts upon detecting the first intersection between the forward and backward search frontiers. After this event, \btit\ adopts all termination conditions from MEET~\cite{MEET}, with the exception of TC2.

\texttt{Neighbors} (Alg.~\ref{alg2}, Lines~1) returns the neighbor set for a newly sampled state. To retain asymptotic optimality (proof omitted), we follow standard RRT*/PRM* practice: either connect to the $k$ nearest neighbors with $k = O(\log n)$, where $n$ is the total number of samples, or connect to all states within radius $r$, with $r = O\!\left((\log n / n)^{1/d}\right)$, where $d$ is the dimension of the configuration space. See~\cite{rrtStar} for additional details.

\texttt{Prune} (Alg.~\ref{alg1}, Lines~{31-33}) iterates over all samples and tree vertices that cannot improve the current best solution and discards them.
 
 % Since MM lacks efficiency in running time, we propose a new stopping condition in this paper. Once $C_{\text{best}} \leq (\fmin ~or ~\frac{\fminf+\fminb}{2})$, \btit terminates the search.
 
 \begin{algorithm}
  %\tiny
  \scriptsize
     $V_\mathcal{F} \leftarrow \si$, $V_\mathcal{B} \leftarrow \sg$, $\cbest \leftarrow \infty$;
     
 $\text{Q}_\mathcal{F} \leftarrow \text{Q}_\mathcal{B} \leftarrow{ \emptyset}$, $E_\mathcal{F} \leftarrow E_\mathcal{B} \leftarrow \emptyset$, $X_{\rm {samples}} \leftarrow \emptyset$;\\
 
\Repeat{Stopped}{

  \If{$\text{Q}_\mathcal{F}=\emptyset$ or $\text{Q}_\mathcal{B} = \emptyset$}{ 
      $X_{\rm samples} \xleftarrow{+}{\texttt{InformedSample}(m,\cbest)}$;\\
         $\text{Q}_\mathcal{F} \leftarrow \{ \si \}$;
   $\text{Q}_\mathcal{B} \leftarrow \{ \sg \}$;\\
  }
    Choose $(\astate) \in \text{Q}_\mathcal{F} \cup \text{Q}_\mathcal{B}$ with $\fmin $;\\

    \eIf{\texttt{Termination}$()$ = \textsc{FALSE} }{
         \eIf{$\astate \in \text{Q}_\mathcal{F}$}{
                         $\astate \xleftarrow{-}{\text{Q}_\mathcal{F}}$;\\
           ${\texttt{Expand}_\mathcal{F}}(\astate)$;
         }{
           //Analogous operation in backward search;
         }
    
    }{
            \texttt{Prune}();
     $\text{Q}_\mathcal{F} \leftarrow \text{Q}_\mathcal{B} \leftarrow \emptyset$;\\
    
    } 
}
\SetKwFunction{Prune}{Prune}
\SetKwProg{myproc}{{\texttt{Prune}}}{}{}
\myproc{$()$}{
   $X_{\rm samples} \xleftarrow{-}\{ \mathbf{x} \in X_{\rm samples} ~\vert~ \widehat{f}(\mathbf{x}) \geq C_{\text{best}}\}$;\\
   $V_\mathcal{F} \xleftarrow{-} \{ \mathbf{x} \in V_\mathcal{F} ~\vert~  \widehat{f}_\mathcal{F}(\mathbf{x}) > C_{\text{best}} \}$;\\
   $E_\mathcal{F} \xleftarrow{-} \{ (\astate,
      \astate') \in E_\mathcal{F} ~\vert~ \widehat{f}_\mathcal{F}(\astate') > C_{\text{best}} \}$;
}  

\SetKwFunction{Termination}{Termination}
\SetKwProg{myproc}{{\texttt{Termination}}}{}{}
\myproc{$()$}{
   \If{\textsc{firstIntersection}}{
     return \textsc{True};
   }
   \If{MEET-termination}{
      return \textsc{True};
   }
   return \textsc{False};
} 

\caption{\small Bidirectional Tight Informed Trees (\btit)}
\label{alg1}
\end{algorithm}

\setlength{\textfloatsep}{0pt}% Remove 
%\vspace{-5mm}

\setlength{\floatsep}{0pt}% Remove \textfloatsep
 %\vspace{-4mm}

\setlength{\floatsep}{0.1cm}% Remove \textfloatsep
 \begin{algorithm}[t]
  \scriptsize
      $V_\mathcal{F} \xleftarrow{+} { \{ \astate\}}$;
   ${X_{\rm samples} \xleftarrow{-}  \{\astate \}}$; \\
    \ForEach{$\astate'\in \texttt{Neighbors}_\mathcal{F}(\astate)$}{
         \If{$\astate'.\mathbf{p}_\mathcal{F} = \emptyset$}{
          $\ghat(\astate') \leftarrow \infty$;
         }  
    
        \If{{$(\astate,\astate') \in E_\mathcal{F}$}}{
                    $\text{Q}_\mathcal{F} \xleftarrow{+} \{\mathbf{x'}\}$;\\
          $\textbf{continue}$;
        }
    
      \If{$\widehat{g}_\mathcal{F}(\astate') > g_\mathcal{F}(\astate) + \widehat{c}(\astate,\astate')$}{
               \If{$\texttt{IsValid}(\astate,\astate') = \textsc{False}$}{
                 \textbf{Continue};
               }
                $E_\mathcal{F} \xleftarrow{-} (\text{prt}_\mathcal{F}(\mathbf{x'}), \mathbf{x'})$;
                          $\text{prt}_\mathcal{F}(\mathbf{x'}) \leftarrow{\mathbf{x}}$;\\
              $E_\mathcal{F} \xleftarrow{+}(\mathbf{x},\mathbf{x'})$; $V_\mathcal{F} \xleftarrow{+}\mathbf{x'} $;\\
              $g_\mathcal{F}(\mathbf{x'}) =g_\mathcal{F}(\mathbf{x}) + c(\mathbf{x},\mathbf{x'})$;
                               \\
              \If{$\astate' \in \mathcal{T}_{\mathcal{B}}~\text{and}~g_\mathcal{F}(\mathbf{x'}) + g_\mathcal{B}(\mathbf{x'})<\cbest$}{
              
                \If{$\texttt{IsValid}(\astate',\text{prt}_\mathcal{B}(\astate')) = \textsc{False}$}{
                    \textbf{continue};
                 }
                 $\cbest = \mathbf{min}(\cbest,g_\mathcal{F}(\mathbf{x'}) + g_\mathcal{B}(\mathbf{x'}))$;
                 
                 \If{\textsc{firstIntersection}}{
                      \textbf{break};
                 }
              }
                ${h'}_\mathcal{F} \leftarrow \max(g_\mathcal{F}(\astate) + \widehat{c}(\astate,\astate'), \widehat{h}_\mathcal{F}(\astate'))$;\\
               \If{$g_\mathcal{F}(\astate) + \widehat{c}(\astate,\astate') + {h'}_\mathcal{F} \leq \cbest$}{      
           $\text{Q}_\mathcal{F} \xleftarrow{+} \{\mathbf{x'}\}$; {$E_\mathcal{F} \xleftarrow{+} (\astate,\astate')$; }  
        }  
      }
    }
\SetKwFunction{Neighbors}{Neighbors}
\SetKwProg{myproc}{{$\texttt{Neighbors}_\mathcal{F}$}}{}{}
\myproc{$(\astate)$}{
        ${\near} \xleftarrow{+} {\texttt{getNN}}(\astate, V_{\mathcal{F}} \cup X_{\text{samples}})$;\\
        {${\near} \xleftarrow{+}\text{prt}_\mathcal{B}(\astate) $};\\
        ${\near} \xleftarrow{+} {\astate.\text{children}}_\mathcal{F}$;\\
     return $\near$;
}
\caption{$\texttt{Expand}_\mathcal{F}(\astate)$}
\label{alg2}
\end{algorithm}

\subsection{Algorithmic Components}

We briefly describe the components used in \btit.
To explain \textbf{C1} (\emph{on-the-fly} heuristic update), consider a configuration $\mathbf{x}$ and assume, w.l.o.g., that $\mathbf{x}$ is generated by the forward search and that $\astate \notin \text{Q}_{\mathcal{B}}$ (the backward-search case is symmetric).
If $\ghat(\mathbf{x}) > \hhat(\mathbf{x})$, we update the heuristic by setting $\hhat(\mathbf{x}) \leftarrow \ghat(\mathbf{x})$.
This update preserves admissibility while preventing the forward search from expanding states that cannot lie on a solution.
Intuitively, since the $\fhat$-values of expanded states form a nondecreasing sequence, a solution is found before expanding any state $\astate$ with $\ghat(\astate) > C_G^*/2$, where $C_G^*$ denotes the cost of the optimal solution in the implicitly defined RGG $G$.
Similar reasoning underlies the state-of-the-art Bi-HS algorithm MEET~\cite{MEET}.

For \textbf{C2} (termination condition), we begin by adopting MEET’s termination criteria~\cite{MEET}. If either termination condition holds, the search terminates for the current batch.

Finally, \textbf{C3} (informed sampling) is a well-established technique for restricting sampling to configurations that can improve the current best cost $C_{\text{best}}$. The subset of configurations capable of yielding such improvements is referred to as the \emph{informed set}~\cite{GammellBS18}. In essence, informed sampling draws samples directly from, or via an indirect transformation of~\cite{YiTGSS18}, the prolate hyperspheroid that characterizes this set.

\section{Analysis}\label{proof}
In this section, we present the proofs for the admissibility of \btit\ ($Theorem~\ref{t1}$), and its asymptotic optimality and probabilistic completeness ($Theorem~\ref{t2}$).  
 Following a similar line in MEET~\cite{MEET}, \btit\ guarantees MM-optimality when the search stops (Alg.~\ref{alg1}, Line 8). Namely, each meeting state is with the lowest $cost$-$to$-$come$ from the forward and backward search trees.

We assume the existence of a random geometric graph $\mathcal{G}_q$, constructed from a given number of informed samples $q$ using a connection strategy (e.g., the neighbor connection used in~\cite{blit}).

\begin{theorem}\label{t1}
When the termination condition is triggered for the current search, \btit\ returns the least-cost solution in $\mathcal{G}_q$ whenever such a solution exists. 
\end{theorem}

\begin{proof}
 As shown in the theory of MEET~\cite{MEET}, MEET ensures the optimality when its termination conditions are triggered in a given graph. Accordingly, \btit\ guarantees finding a minimal-cost solution in $\mathcal{G}_q$ if it exists.   
\end{proof}

\begin{theorem}\label{t2}
\btit\ is asymptotically optimal and probabilistically complete.
\end{theorem}

\begin{proof}
 According to Theorem~\ref{t1}, the cost \( C_{\text{best}} \) returned by any instance of \btit, converges almost surely to the optimal cost \( C(\sigma^*) \) as $q$ approaches to infinity. That is,
\[
        \mathbb{P} \left(\lim_{q\to\infty}\cbest(\pi_{q,\btit} ) =C_{\text{best}} = C(\pi^*)\right) = 1,
   \]
Therefore, \btit\ is asymptotically optimal. Since it returns a solution whenever one exists (i.e., \( C_{\text{best}} < \infty \)), it is also probabilistically complete.
\end{proof}
%\begin{itemize}

%\item Use a zero before decimal points: Ò0.25Ó, not Ò.25Ó. Use Òcm3Ó, not ÒccÓ. (bullet list)
%\end{itemize}
%\subsection{Initial Solutions of A Forward, Asymmetric Bidirectional, and Bidirectional Sampling-Based Path Planner } 

%If the non-mutual $k$-nearest neighbors strategy is used to find a valid path by a forward (BIT*), asymmetric bidirectional (AIT*/EIT*), and bidirectional search (\btit) with an admissible heuristic, the initial found solutions can be different. It is known that AIT*/EIT* provides a backward parent and \btit supplies a backward and forward parent. We show the cases as follows:

%Case1: A forward search finds a path from a given RGG. The initial solution found by asymmetric bidirectional or bidirectional search should not be larger than the forward search.

%Case2: A forward search cannot find an initial solution in a given RGG. The asymmetric bidirectional or the bidirectional search can still find an initial solution, leading to a longer initial path length. The right upper figure of Fig.~\ref{unhcm} shows that an initial solution found with non-mutual $k$-nearest neighbors by BIT* (1597.32), EIT* (1596.7), and \btit (1549.01) on UNH Campus Map. The path costs display that \btit returns a better initial solution than BIT* and EIT*.

\section{Experimental Results}

  \begin{figure}[h]
  % \vspace{-2mm}
    \centering
         \scalebox{0.8}{    \includegraphics[width=\linewidth,scale=1]{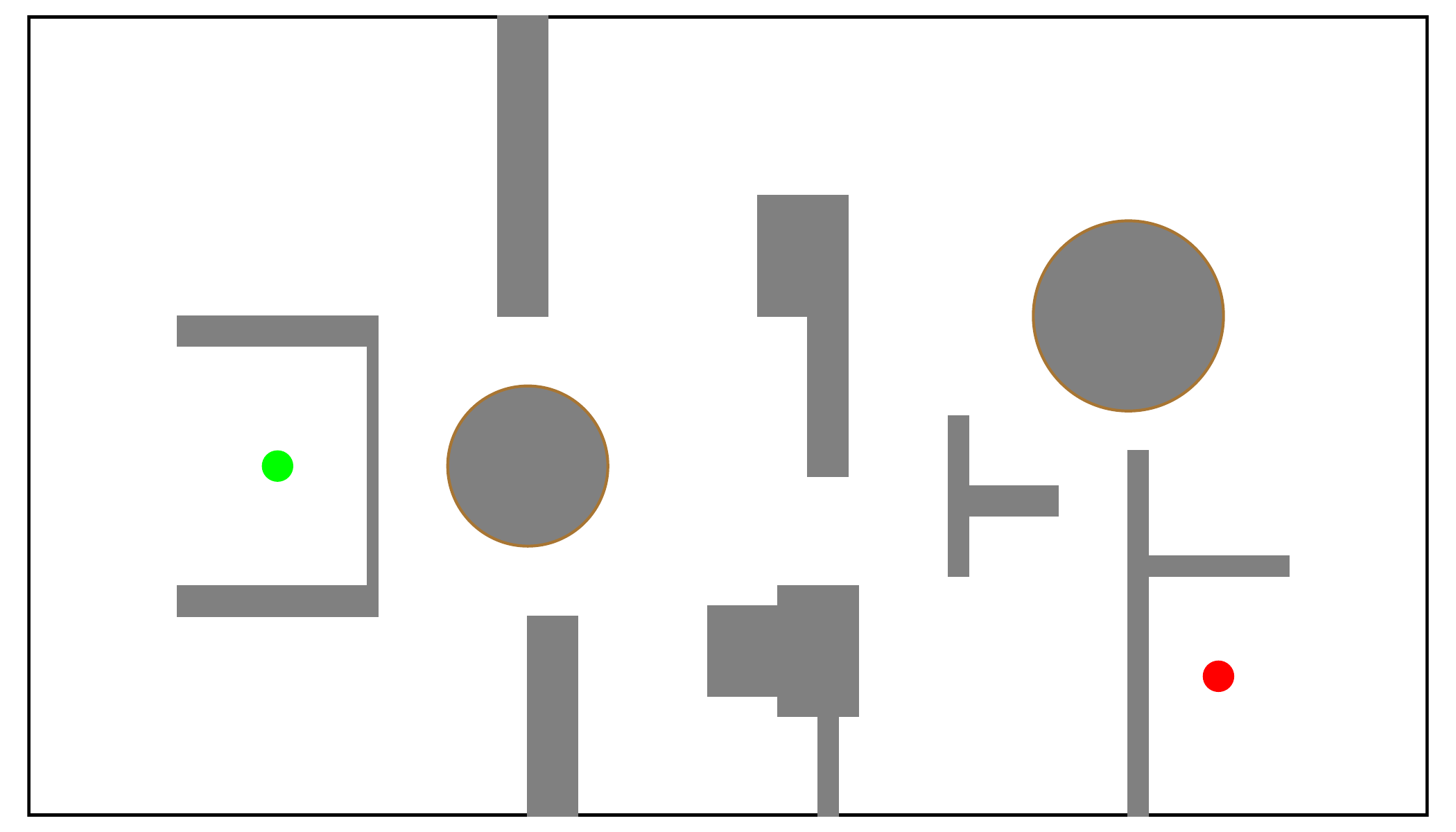}}
\caption{Experimental environment for the double-integrator robot (DIR) in a 2D workspace. This environment features enclosures, narrow passages, and multiple topological routes.}
    \label{DI}
\end{figure}

  \begin{figure}[h]
  % \vspace{-2mm}
    \centering
    \scalebox{0.64}{    \includegraphics[width=\linewidth,scale=1]{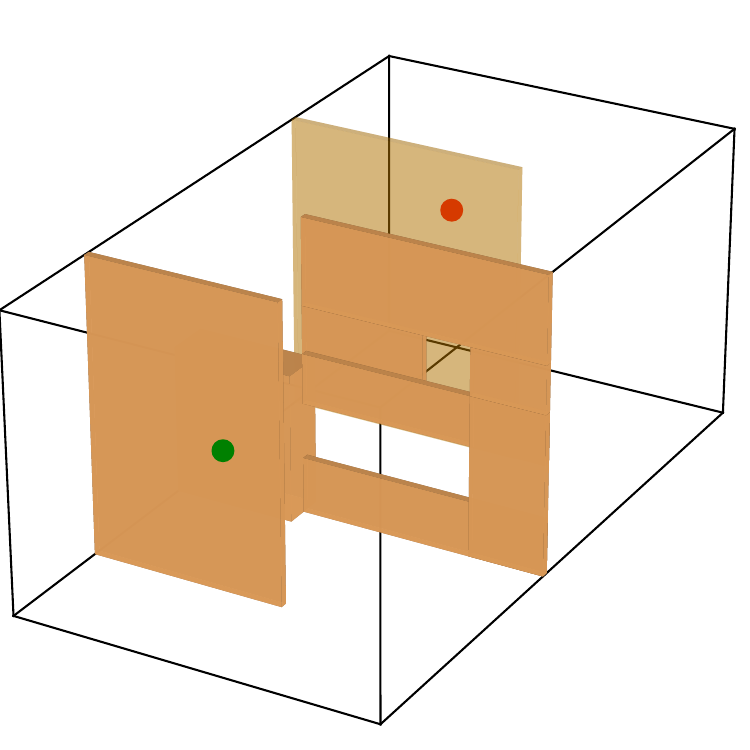}}

\caption{Experimental environment for the 10D linearized quadrotor (LQ) system in a 3D workspace. This environment features with non-convex corridors and narrow passages.}
    \label{lq}
\end{figure}

 \begin{figure*}
    \centering
    \subfloat[Double Integrator Robot ($\mathbb{R}^{4}$)]{
        \scalebox{0.48}{\includegraphics[width= \textwidth]{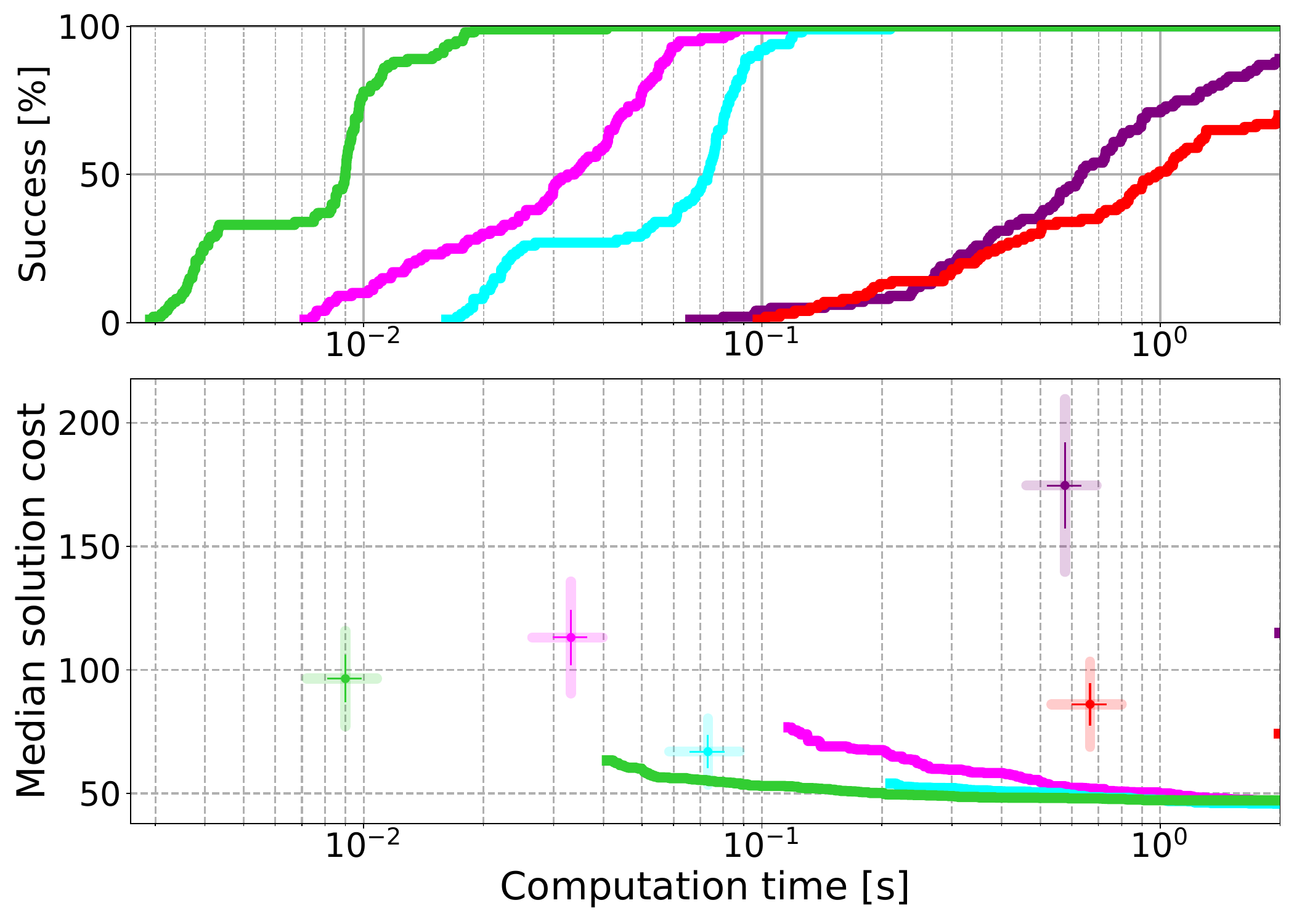} \label{fig:wall} }
    }
        \subfloat[Linearized Quadrotor ($\mathbb{R}^{10}$)]{
        \scalebox{0.48}{\includegraphics[width= \textwidth]{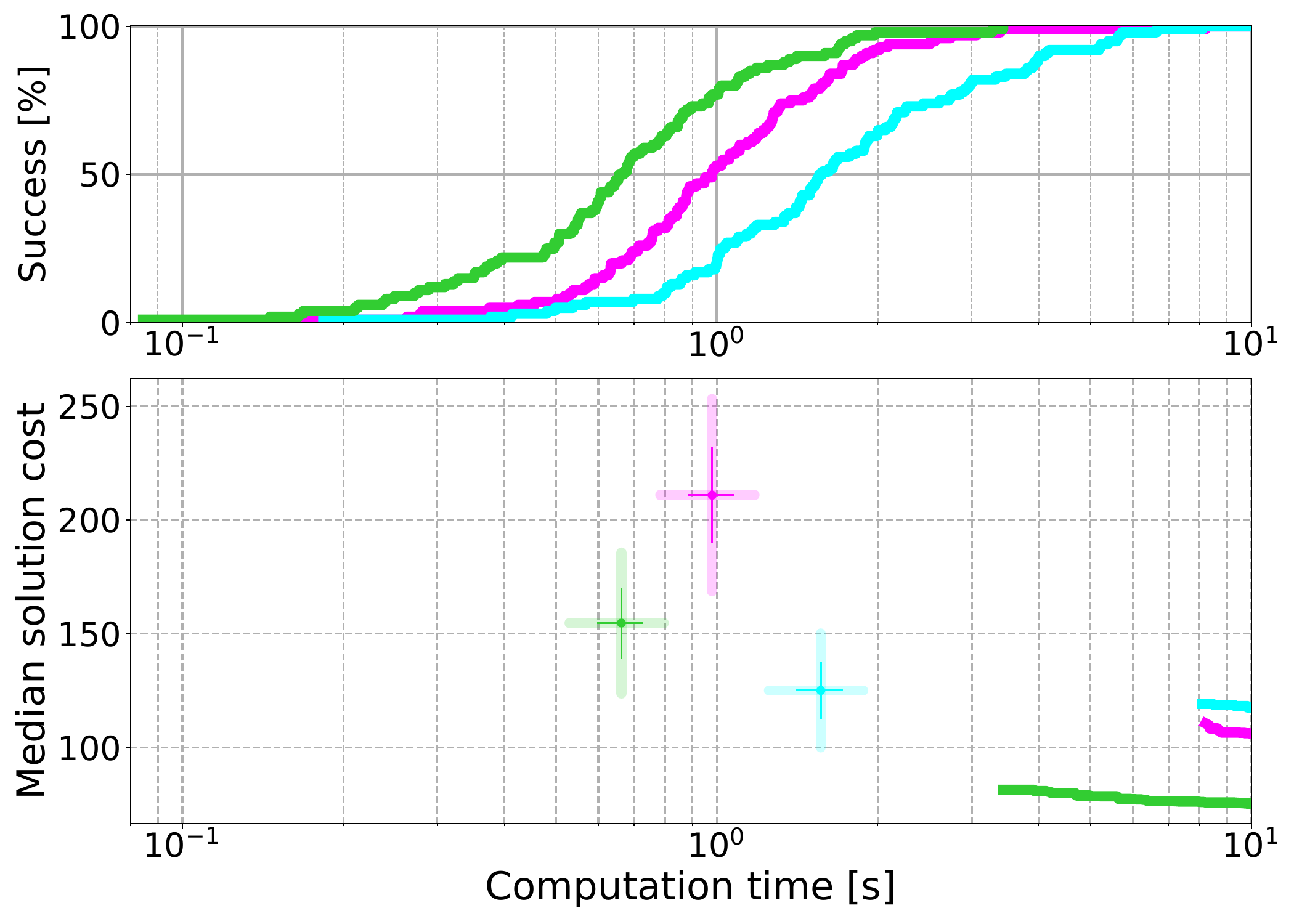} \label{fig:ex2} }
    }
    
    \includegraphics[width= \textwidth]{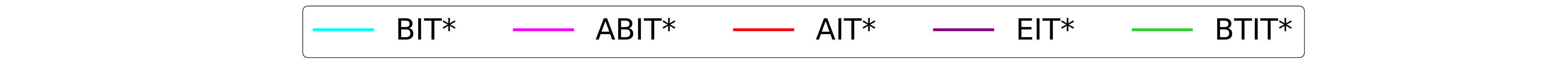}
    % \subfloat[Wall Gap In $\mathbb{R}^8$]{
    % %\scalebox{0.9}{
    %   \includegraphics[width=160mm,scale=1]{NewPaper.pdf}}
    %}  % start with (b)
  
%     \subfloat[Wall Gap In $\mathbb{R}^{16}$]{
%     %\scalebox{0.63}{
% \includegraphics[width=160mm,scale=1]{R16WALL.png}}
    %}
    \caption{Comparative performance of evaluated planners across different domains in terms of median solution cost and success rate over time, with 95$\%$ confidence interval in each plot. Note that the median solution cost is $\infty$ when the success rate of a planner is below 50$\%$. The dots in each plot denote median initial solution times and costs with 95$\%$ confidence intervals (CI), while shadows around the dots indicate a more conservative 99$\%$ CI.}
     
\end{figure*}

%  \begin{figure}[t]
%     \centering 
%     %\subfloat[]{}
%     \scalebox{0.9}{
%     \includegraphics[width=90mm,scale=0.8]{UNH.pdf}}

%     \caption{Initial and final search performances on selected Campus Map. We set cut = 0 to restrict the violin range within the range of our testing data. The violin plots display the 95$\%$ confidence interval (lines that extend from the center). The results show \btit\ outperforms EIT* in faster finding an initial solution and in prompting the convergence.}
%     \label{fig:VPUNH}  
% \end{figure}

We implement \btit\ in Open Motion Planning Library (OMPL)~\cite{sucan2012open} and compare its performance with other batch-wise sampling-based motion planners in OMPL, including BIT*, AIT, EIT*, and ABIT*. Our experimental evaluations across two kinodynamic systems: a Double Integrator Robot (4-D DIR) and a 10D Linearized Quadrotor (10-D~LQ). All parameters align with those presented in~\cite{EIT}. The $r$-disk-nearest neighbors connection strategy is used for the testing domains.
Each planner performs 100 trials, sampling per batch with allocated time per trail as follows:  200 states with 2s in LAB (4-D~DIR) and 300 states with 10s in LAB2 (10-D~LQ). To calculate the optimal trajectory between states, the fixed-final-state free-final-time optimal controller~\cite{kinorrt*} is used as a \emph{steering} function for kinodynamic systems. The heuristic value and the admissible edge cost are derived from the optimal controller. Notably, the green and red dots denote \si and \sg in each figure, respectively.  
All planners are run on a desktop with 16GB of RAM and Intel I7-7700k CPU running Ubuntu 20.04 and are implemented in C++ while and each undergoes 200 segments collision checks between states, maintaining consistent solution quality even when increased to 10,000. 
\subsection{2D-Double integrator robot}

The double-integrator robot (DIR) benchmark~\cite{webb2013kinodynamic,kinorrt*} is adopted as a simplified yet representative kinodynamic model of robotic motion: it exhibits second-order dynamics in which the control input is acceleration, requiring the planner to reason jointly over position and velocity evolution. The state is four-dimensional, $x=[l^{T}\;v^{T}]^{T}$, where $l\in\mathbb{R}^{2}$ denotes planar position and $v\in\mathbb{R}^{2}$ denotes planar velocity. All planners are evaluated under identical conditions in a $14\times 8~\mathrm{m}^{2}$ indoor laboratory workspace with bounded velocities $v\in[-2,2]^{2}~\mathrm{m/s}$ (Fig.~\ref{DI}), which constrains admissible motions and emphasizes dynamic feasibility. 

The lab setup is designed to stress SBMP with heuristic search via start and goal enclosures, narrow passages, irregular obstacle layouts, sharp turns, dead ends, and multiple homotopy classes. As the RGG densifies, increased connectivity inflates the branching factor, raising search effort and computational cost.

% The workspace is designed to stress sampling-based motion planning (SBMP) with heuristic graph search through several compounding challenges: (i) start and goal enclosures that restrict feasible departure and arrival directions, (ii) narrow passages that demand precise dynamic regulation for collision-free traversal, (iii) irregular obstacle layouts that degrade simple distance-to-go heuristics and induce locally attractive but globally suboptimal expansions, (iv) sharp turns that are difficult to negotiate under bounded velocity and acceleration, (v) dead-end structures that amplify unproductive exploration and backtracking, and (vi) multiple homotopy classes that yield diverse topological solution families. These challenges intensify as the random geometric graph (RGG) is densified: higher sampling density increases the number of vertices and candidate neighbor transitions, inflating the effective branching factor and thereby increasing search effort and computational cost.

% Additionally, this domain requires accurate handling of sharp turns and more complex planning for direction changes due to the velocity constraints. 

As demonstrated in Fig~\ref{DI}, EIT* ($89\%$) and AIT* ($70\%$) fail to reach a $100\%$ success rate compared to the other planners. Moreover, according to the defined metrics, \btit\ achieves a median initial search time that is $3.6\times$ and $8.1\times$ faster than ABIT* and BIT*, respectively.  

\subsection{10-D Linearized Quadrotor}
In Lab2 (Fig.~\ref{lq}), we evaluate a linearized quadrotor model as described in~\cite{kinorrt*}. The system has a 10D state space,
$x=[l^{T}\;v^{T}\;r^{T}\;w^{T}]^{T}$, comprising 3D position $l\in\mathbb{R}^{3}$, 3D velocity $v\in\mathbb{R}^{3}$, a 2D attitude parameterization $r\in\mathbb{R}^{2}$, and 2D angular velocity $w\in\mathbb{R}^{2}$.\footnote{\scriptsize For all planners, we discretize each state component within the following bounds: $l$ over an $18\times 10\times 8~\mathrm{m}^{3}$ workspace, $v\in[-2.5,2.5]^{3}~\mathrm{m/s}$, $r\in[-1.5,1.5]^{2}~\mathrm{rad}$, and $w\in[-4,4]^{2}~\mathrm{rad/s}$.} This domain is designed for the 10-D LQ setting and includes narrow, non-convex passages that induce connectivity bottlenecks. These bottlenecks both hinder sampling and degrade the accuracy of standard heuristics. As the RGG densifies, the resulting increase in branching and edge evaluations substantially amplifies computational cost.

% Lab2 is tailored to the 10D LQ setting and includes non-convex corridors and narrow passages (as discussed above) that induce severe connectivity bottlenecks. The concave geometry produces a systematic mismatch between heuristic guidance and feasible kinodynamic progress: heuristics based on geometric proximity or unconstrained LQ cost-to-go often favor states near concave boundaries, while valid trajectories must follow the corridor through a sequence of intermediate states that may temporarily increase heuristic cost. As a consequence, heuristic-guided search can spend substantial effort on locally attractive expansions near corridor entrances and exits before discovering the small set of gateway states that traverse the corridor. These effects are amplified as the RGG is densified, since open regions contribute many near-equivalent neighbors, inflating the branching factor and increasing the number of expansions and collision checks, leading to higher runtime.

 Compared to other planners, AIT* and EIT* fail due to the multiple queries to NN and frequent edge-cost calculations. 
As shown in Fig.~\ref{lq}, \btit\ finds the median initial solution 1.48$\times$ faster than ABIT*, while \btit\ achieves a 2.37$\times$ speedup compared to BIT*. 
%An experimental test video is found at  \textcolor{blue} {}.

 \section{Conclusions and Future Work}
In this paper, we present \btit, an asymptotically optimal sampling-based planner that leverages bidirectional heuristic search with a tighter termination condition. \btit\ performs bidirectional heuristic search on an RGG over informed states while maintaining and expanding edge frontiers. It also introduces \btit\, an anytime-tight termination condition for Bi-HS, yielding progressive improvements in search efficiency and sampling focus. Empirical results on realistic robotic models show that \btit\ outperforms representative non-lazy batch-wise SBMP planners in both time to first solution and convergence on the evaluated domains.

Future research will focus on leveraging bidirectional search for kinodynamic motion planning in complex environments, supported by experiments \cite{mu2016design}.

% Looking forward, motivated by the challenges of kinodynamic motion planning under nonlinear dynamics and dynamic obstacles, we aim to leverage bidirectional search effort to better guide forward search in these settings and to conduct experiments \cite{mu2016design}.

%\section*{Acknowledgement}
%We would like to  express our sincere gratitude to Prof. Wheeler Ruml for his comments on drafts of this work, and Debarpan Bhowmick for his help with experiments on Pixhawk 5X drone flight tests. We also want to thank the anonymous reviewers for their valuable feedback and suggestions.
      
%\addtolength{\textheight}{-12cm}   

%%%%%%%%%%%%%%%%%%%%%%%%%%%%%%%%%%%%%%%%%%%%%%%%%%%%%%%%%%%%%%%%%%%%%%%%%%%%%%%%

%%%%%%%%%%%%%%%%%%%%%%%%%%%%%%%%%%%%%%%%%%%%%%%%%%%%%%%%%%%%%%%%%%%%%%%%%%%%%%%%

%%%%%%%%%%%%%%%%%%%%%%%%%%%%%%%%%%%%%%%%%%%%%%%%%%%%%%%%%%%%%%%%%%%%%%%%%%%%%%%%
%\section{Discussion}
   %In the real world, the searching environments are not always ideally sparse, which could make \btit\ fails, so we are planning to solve this issue in the future work va.

%\section*{ACKNOWLEDGE}

%%%%%%%%%%%%%%%%%%%%%%%%%%%%%%%%%%%%%%%%%%%%%%%%%%%%%%%%%%%%%%%%%%%%%%%%%%%%%%%%
\bibliographystyle{IEEEtran}
\bibliography{ICRA}

\end{document}